\documentclass[preprint,12pt,authoryear,nopreprintline]{elsarticle}

\journal{Transportation Research Part C: Emerging Technologies}
\date{}
\biboptions{authoryear,round}

\usepackage{amsmath,amssymb,amsfonts,amsthm}
\usepackage{booktabs}
\usepackage{graphicx}
\usepackage{algorithm}
\usepackage{algorithmic}
\usepackage{multirow}
\usepackage{xcolor}
\usepackage{subcaption}
\usepackage{enumitem}
\usepackage{colortbl}
\usepackage{seqsplit}
\usepackage[hidelinks]{hyperref}
\usepackage{cleveref}

\makeatletter
\def\ps@pprintTitle{%
  \let\@oddhead\@empty
  \let\@evenhead\@empty
  \let\@oddfoot\@empty
  \let\@evenfoot\@empty}
\makeatother

\newtheorem{proposition}{Proposition}

\newcommand{\R}{\mathbb{R}}
\newcommand{\E}{\mathbb{E}}
\DeclareMathOperator{\clip}{clip}
\DeclareMathOperator{\sigmoid}{sigmoid}

\crefname{assumption}{Assumption}{Assumptions}
\crefname{proposition}{Proposition}{Propositions}
\crefname{theorem}{Theorem}{Theorems}
\crefname{lemma}{Lemma}{Lemmas}
\crefname{remark}{Remark}{Remarks}
\crefname{algorithm}{Algorithm}{Algorithms}
\crefname{equation}{Eq.}{Eqs.}

\begin{document}

\begin{frontmatter}

\title{Completion-Aware Cross-Fidelity Offline-to-Online Reinforcement Learning for Multi-Line Bus Holding}

\author[csu]{Yifan Zhang}
\ead{erzhu419@gmail.com}
\ead{204201048@csu.edu.cn}
\author[csu]{Liang Zheng}
\ead{zhengliang@csu.edu.cn}
\author[monash]{Qifan Zhang}
\ead{zhangqifan0419@gmail.com}
\address[csu]{Central South University, Changsha, Hunan, China}
\address[monash]{Faculty of Information Technology, Monash University Malaysia, Malaysia}

\begin{abstract}
Exploratory reinforcement learning (RL) on an operating bus fleet is
impractical, while policies trained only from historical data cannot acquire
new experience. Hybrid Offline-and-Online (H2O) RL combines fixed target replay
with simulator interaction, but the inexpensive online simulator can differ
from the target in transition and event-duration dynamics. We study this
cross-fidelity problem for multi-line bus holding and address a failure mode in
which lower generalized passenger time coexists with incomplete passenger
journeys.

We formulate asynchronous holding as a 24,000-s chronological semi-Markov
decision process (SMDP) with exact passenger-time accounting and a mixed
discrete--continuous action. An immutable Simulation of Urban MObility (SUMO)
dataset contains 3,284,718 transitions from 240 episodes, eight behavior
policies, and 30 environment cells. An improved H2O framework (H2O+) actor is
adapted to transit, initialized by implicit Q-learning (IQL), and continued for
25,000 online-simulator decision
events. The deployed policy then applies a preregistered completion-safety
reserve: in low-load tail-bus states with remaining demand, it raises the
physical hold to 60 s. Its direct parent uses the identical actor and inherited
action transform but omits this final reserve.

Formal confirmation used 1,350 fresh SUMO episodes: ten independent
environment blocks, three demand levels, five training seeds, and 45 fixed
physical behaviors. Averaged over the candidate--parent pairs, the reserve-augmented policy
reduced generalized passenger time from 2,484.86 to 2,392.73 s per departed
passenger (difference $-92.12$ s; $-3.71\%$), increased completion from
0.96185 to 0.96650 (difference $+0.00466$), and reduced unfinished passengers
from 655.42 to 578.20. It improved the primary endpoint and satisfied a 0.003
completion noninferiority margin in all ten paired blocks; both one-sided exact
sign-test p-values were 0.0009766, so the preregistered intersection--union test
passed. Benefits had the same direction at all demand levels and training
seeds. The confirmed claim is limited to this fixed candidate versus its fixed
direct parent in simulator-to-simulator endpoint evaluation.
\end{abstract}


\end{frontmatter}

\section{Introduction}
\label{sec:intro}

High-frequency bus service is vulnerable to bunching: a delayed vehicle faces
more waiting passengers and longer dwell, while its follower faces fewer,
causing an initially small headway perturbation to grow
~\citep{rezazada2024review}. Holding an early bus at selected stops can restore
spacing, but it also adds time for passengers already on board
~\citep{delgado2012holding}. Analytical headway control and schedule-reliability
control address different operational goals
~\citep{daganzo2009headway,xuan2011dynamic}; RL can represent richer stochastic
and multi-line interactions
~\citep{alesiani2018reinforcement,wang2020real,wang2023multiobjective}.

Learning through exploration on an operating fleet is generally unacceptable.
Offline RL can use historical transitions, but its policy is vulnerable outside
logged action support. Online learning in a simulator supplies adaptive
experience, but an inexpensive model can differ from the target in travel
times, dwell propagation, demand events, and the physical duration between
asynchronous decisions. H2O RL combines immutable
target replay with interaction in an imperfect simulator~\citep{niu2022h2o};
H2O+ generalizes this construction across offline and online base learners
~\citep{niu2025h2oplus}. The transit question is therefore not only whether
simulator interaction improves an offline policy, but whether the resulting
behavior remains service-complete when transferred back to the target
simulator.

Bus holding makes this distinction consequential. Stop decisions are
asynchronous, several buses can request control at the same physical time, and
exact zero hold is much more common than a positive duration. More importantly,
an objective that integrates waiting and in-vehicle passenger time does not by
itself ensure that journeys finish before a fixed evaluation horizon. During
development we observed precisely this objective--constraint mismatch: lower
generalized passenger time could coexist with a completion-rate loss large
enough to fail a preregistered service boundary. The final policy therefore
separates learning from deployment. A frozen H2O+ actor supplies the state
dependent action, while a deterministic completion-safety reserve modifies
only a bounded set of low-load tail-bus states with remaining demand.

The resulting evidence chain is summarized in \Cref{tab:evidence_chain}. The
target is a SUMO microsimulation of 12
directional services and 389 scheduled vehicle trips. The offline replay
contains 3,284,718 chronological transitions from 240 complete episodes. The
online environment is a causal event-driven simulator with matched timetable,
capacity, passenger-journey, state, action, and reward contracts. The final
candidate and direct parent share five frozen terminal-quarter actors and differ
only in the last deterministic completion-safety transform; lineage controls
retain their own frozen actors. Candidate design used prior development
outcomes, but the final candidate, direct parent, analysis, and ten-block formal
environment were fixed before formal outcomes were accessed.

\begin{table*}[t]
\centering
\small
\caption{Separation of learning, adaptive development, and formal confirmation.
Only the final column supplies confirmatory evidence.}
\label{tab:evidence_chain}
\begin{tabular}{>{\raggedright\arraybackslash}p{0.29\textwidth}
                >{\raggedright\arraybackslash}p{0.29\textwidth}
                >{\raggedright\arraybackslash}p{0.32\textwidth}}
\toprule
\textbf{Frozen learning evidence} & \textbf{Adaptive development} &
\textbf{Fresh formal confirmation} \\
\midrule
Target-SUMO replay; calibrated online simulator; five IQL-initialized H2O+
actors; 25,000 simulator events; fixed endpoint checkpoints. &
Terminal actor taper and deterministic hold transforms were developed over
successive preregistered rounds. Development outcomes informed the final
completion-safety reserve and therefore are not confirmatory. &
The fixed candidate and fixed direct parent were evaluated on 30 untouched
SUMO cells: ten independent blocks crossed with three demand levels. Exact
paired block tests used all five training seeds within each block. \\
\bottomrule
\end{tabular}
\end{table*}

Formal confirmation comprised 1,350 new SUMO episodes and 45 fixed physical
behaviors. Relative to its direct behavior parent, the reserve-augmented policy
reduced generalized passenger time by 92.12 s per departed passenger
($3.71\%$), increased completion by 0.00466, and reduced unfinished passengers
by 77.22 on average. The primary endpoint improved and the completion
noninferiority condition passed in all ten independent blocks. The result is
not attributed to density-ratio correction: held-out ratio diagnostics did not
generalize, and the sealed actor training used unit simulator weights.

This work makes four contributions. First, it defines a chronological
finite-horizon SMDP representation that serializes simultaneous
controls while charging exact generalized passenger time once per physical
interval. Second, it provides an audited multi-policy target dataset and a
causal online simulator under a common transit contract. Third, it combines a
stabilized H2O+ actor with an explicit, state-conditioned completion-safety
reserve whose direct effect can be isolated against an identical parent.
Fourth, it supplies a preregistered formal comparison on fresh environment
blocks, with the independent unit, noninferiority margin, and claim boundary
fixed before outcome access. The formal claim concerns the final reserve
relative to its identical direct parent in this simulator scenario family;
\Cref{sec:discussion:limitations} collects the limits on broader interpretation.

\section{Related Work}
\label{sec:related}

The study connects offline-to-online RL, transfer across transition dynamics,
and bus holding. We distinguish the algorithmic lineage from the mechanism
that is actually confirmed in the final experiment.

\subsection{Offline and offline-to-online reinforcement learning}

Soft actor-critic (SAC) provides the entropy-regularized, off-policy backbone
used by the online learners in this study~\citep{haarnoja2018soft}. Its later
formulation adds automatic entropy-temperature adjustment
~\citep{haarnoja2018applications}.
Conservative Q-learning (CQL) depresses values for actions outside offline
support~\citep{kumar2020cql}; IQL instead fits an
expectile state value and extracts an advantage-weighted policy without
evaluating unseen actions in its fitted target~\citep{kostrikov2022iql}. We use
IQL to initialize the frozen H2O+ actors and SAC-style updates for their online
continuation.

Offline-to-online fine-tuning can destroy a useful offline policy when early
online data shift the state--action distribution and bootstrap errors
propagate. Balanced replay with a pessimistic critic ensemble is one
stabilization route: it prioritizes online and near-on-policy offline samples
~\citep{lee2022balanced}, rather than simply fixing a 50/50 replay split.
Calibrated Q-learning (Cal-QL) constrains conservative values to a reference
policy's scale~\citep{nakamoto2023calql}. Policy expansion (PEX) retains an
offline policy and adaptively composes it with a newly learned policy
~\citep{zhang2023pex}. Reinforcement learning with prior data (RLPD) mixes
prior and newly collected replay while learning from random initialization
~\citep{ball2023efficient}, whereas warm-start reinforcement learning (WSRL) begins
from an offline policy, collects a frozen-policy warmup buffer, and then trains
from online replay~\citep{zhou2025wsrl}. These methods motivate the stabilization
controls in the development lineage, but the formal comparison in this paper
does not claim a universal ranking among their source algorithms.

Reward-guided conservative Q-learning (RG-CQL) has combined offline training
and online fine-tuning for coordinating ride-pooling with public transit
~\citep{hu2025rgcql}. We therefore make no broad novelty claim for
offline-to-online RL in transportation. Our focus is an asynchronous fixed-route
holding problem with a target-simulator dynamics gap and an explicit completion
boundary.

\subsection{Transfer across dynamics}

Off-dynamics RL corrects source-domain rewards using a conditional
target-to-source dynamics log ratio estimated by domain classifiers
~\citep{eysenbach2021off}. H2O combines a fixed target-domain
dataset with interaction in an imperfect simulator and uses dynamics-aware
correction when simulator transitions are poorly supported by target dynamics
~\citep{niu2022h2o}. H2O+ broadens that construction across offline and online
base learners~\citep{niu2025h2oplus}.

Our adaptation includes physical event duration in a joint/marginal classifier,
but the final training profile uses unit simulator weights. Thus the classifier
is a domain-shift diagnostic, not the confirmed mechanism. Its empirical
validation is reported in \Cref{sec:exp:validity,app:ratio}; the formal comparison
isolates the action reserve applied to a fixed learned actor.

\subsection{Bus holding and service completion}

The bus-bunching literature spans demand, supply, and control mechanisms
~\citep{rezazada2024review}. \citet{daganzo2009headway} studies headway-based
holding and its stability, whereas \citet{xuan2011dynamic} uses arrival
deviations from a virtual schedule to improve schedule reliability. These are
distinct control objectives. \citet{delgado2012holding} compares holding with
boarding limits under passenger-time objectives. Real-time optimization
incorporates changing demand and running times
~\citep{sanchezmartinez2016realtime}, while predictive control combines
headway forecasts with dynamic holding~\citep{andres2017predictive}.

RL formulations include high-frequency single-service control
~\citep{alesiani2018reinforcement}, multi-agent dynamic holding
~\citep{wang2020real}, asynchronous multi-agent coordination
~\citep{wang2021asynchronous}, and shared-corridor multi-line control
~\citep{wang2023multiobjective}. Asynchronous bus control is therefore not
itself a new contribution. Our study instead combines immutable
target replay, online interaction in a separate reduced-order simulator,
adaptive development of a fixed action transform, and final confirmation on
fresh target-simulator cells. The common passenger ledger and the paired final
reserve comparison distinguish this evaluation from a generic claim that a
new RL learner outperforms existing bus-holding algorithms.

The completion-safety reserve is intentionally simple. It is not a new generic
safe-RL algorithm and does not certify arbitrary operational constraints.
Instead, it encodes one observed transit requirement: when a lightly loaded
tail bus has no follower, a leader exists, and demand remains, suppressing hold
can improve accrued time while leaving passengers unfinished. The paired parent
comparison changes only the reserve, allowing this concrete mechanism to be
tested without reselecting a checkpoint or retraining the actor.

\section{Problem formulation}
\label{sec:problem}

We model multi-line bus holding using a finite-horizon SMDP representation
~\citep{sutton1999between}. The citation supplies the semi-Markov foundation for
variable-duration decisions; the serialization and passenger-time ledger below
are study-specific. The formulation
has three properties that are essential for
the study: decisions occur asynchronously at bus-stop events, simultaneous
decisions are serialized without duplicating system cost, and the training
reward is exact generalized passenger time rather than a headway proxy.

\subsection{Chronological decision process}
\label{sec:problem:smdp}

One shared policy controls the bus fleet. Let $\mathcal{L}$ be the set of
directional services and $\ell\in\mathcal{L}$ a service index. Physical time
advances continuously, but control is
requested only when a bus reaches a non-terminal stop. At physical decision
time $t_m$, all controllable arrivals form an ordered batch
$\{e_{m,1},\ldots,e_{m,n_m}\}$, where $m$ indexes physical batches, $n_m$ is
the batch size, and $e_{m,j}$ is the event in position $j$. Events are ordered
deterministically by service and vehicle identifier. The registered schema
supports $1\leq n_m\leq16$.

Simultaneous events are converted into an agent-environment cycle (AEC). Each
decision token is uniquely identified by a pair $(m,j)$, and $k$ enumerates
these pairs in chronological AEC order. Every token observes the entire fixed
batch and the actions already selected in that batch; the environment advances
only after the final action is available. Let
$x^{\mathrm{raw}}_{m,j}\in\R^{16}$ be the 16-feature raw SUMO event observation
for $e_{m,j}$, where $\R$ denotes the real numbers. Removing physical time and
the two system-wide passenger counts gives the 13-feature deduplicated local
event vector $x_{m,j}\in\R^{13}$. Let $N_{\mathrm{w}}(t_m)$ and
$N_{\mathrm{v}}(t_m)$ be the system-wide waiting and in-vehicle passenger
counts at batch time $t_m$. The 245-dimensional raw packed state and the
229-dimensional canonical state are
\begin{equation}
\begin{aligned}
s^{\mathrm{raw}}_k
&=\operatorname{pack}\!\left(
t_m,n_m,j,N_{\mathrm{w}}(t_m),N_{\mathrm{v}}(t_m),
\right.\\[-0.25em]
&\hspace{5.2em}\left.
x_{m,1},\ldots,x_{m,n_m},\text{prior actions}
\right)\in\R^{245},\\
s_k&=\operatorname{normalize}(s^{\mathrm{raw}}_k)\in\R^{229}.
\end{aligned}
\label{eq:v7_state}
\end{equation}
Here $x^{\mathrm{raw}}_{m,j}$ contains the service and stop indices, physical
time, forward and backward headways and their scheduled targets, eligible
stranded and on-board passenger counts, eligible passenger-arrival rate, base
dwell time, forward/backward vehicle-presence flags, vehicle capacity, and the
two system-wide passenger counts. The packing operation stores the five-value
raw batch prefix once, records active slots, inserts the exact actions for
earlier slots, and retains the no-hold sentinel for the current and future
slots. The normalization operation converts each 13-feature local event vector
to 12 dimensionless local features, using capacity to scale passenger counts
and demand instead of retaining it as a separate coordinate; together with the active flag and prior
action, each canonical slot has width 14. Inactive slots are masked.
The raw width is $5+16(1+13+1)=245$, and the canonical width is
$5+16(1+12+1)=229$. Here \(\operatorname{pack}\) denotes this fixed padding
and concatenation, and \(\operatorname{normalize}\) denotes the registered
feature scaling and local feature reduction.
This representation is centralized; its sufficiency as a Markov state is a
modeling assumption, discussed in \Cref{sec:discussion:limitations}.

\paragraph{Action.}
The policy emits one scalar $a_k\in[-1,1]$. Its physical hold duration is
\begin{equation}
\tau_k=\frac{a_k+1}{2}\,\tau_{\max},
\qquad \tau_{\max}=60\ \mathrm{s},
\label{eq:v7_action}
\end{equation}
where $\tau_k$ is the hold applied at token $k$ and $\tau_{\max}$ is the
operational cap. The exact action $a_k=-1$ means zero hold; all actions in
$(-1,1]$ mean strictly positive holds. Thus one mathematical action has one
physical meaning, rather than using separate gate and duration coordinates.

\paragraph{Duration and continuation.}
Let $\Delta t_k\geq0$ be physical time elapsed from token $k$ to token $k+1$,
and let $z_k\in\{0,1\}$ indicate that the transition is terminal. The stored
continuation factor is
\begin{equation}
\Gamma_k=1-z_k.
\label{eq:v7_discount}
\end{equation}
The horizon is $T_{\mathrm{hor}}=24{,}000$ s and the objective is
undiscounted. Consequently, $\Gamma_k=1$ for every non-terminal transition and
$\Gamma_k=0$ only at the absorbing endpoint. For two consecutive tokens inside
one simultaneous batch, $\Delta t_k=0$; the final token of the batch carries
the complete positive duration to the next physical batch.

\subsection{Exact generalized passenger-time reward}
\label{sec:problem:reward}

Let $N_{\mathrm{w}}(t)$ and $N_{\mathrm{v}}(t)$ denote the numbers of waiting
and in-vehicle passengers at physical time $t$, respectively. For token $k$ in
physical batch $m$, generalized passenger-time cost is
\begin{equation}
C_k=
\omega_{\mathrm{w}}\int_{t_m}^{t_m+\Delta t_k}N_{\mathrm{w}}(u)\,\mathrm{d}u
+
\omega_{\mathrm{v}}\int_{t_m}^{t_m+\Delta t_k}N_{\mathrm{v}}(u)\,\mathrm{d}u,
\qquad
(\omega_{\mathrm{w}},\omega_{\mathrm{v}})=(2,1),
\label{eq:v7_cost}
\end{equation}
where $u$ is the integration variable, $C_k$ is measured in generalized
passenger-seconds, and
$\omega_{\mathrm{w}}$ and $\omega_{\mathrm{v}}$ are the fixed waiting- and
in-vehicle-time values. The stored dataset reward is
\begin{equation}
r_k=-C_k.
\label{eq:v7_reward}
\end{equation}
Thus $r_k$ is stored in physical generalized passenger-seconds; numerical
conditioning applied by a trainer is not part of the data contract.
\Cref{eq:v7_cost,eq:v7_reward} define one reward for every token. Because the ledger
integrates synchronized system passenger counts, the chronological sum of
$C_k$ is non-overlapping. Intermediate AEC tokens have zero duration and hence
zero reward; system cost appears exactly once between physical batches.
\Cref{prop:ledger} in \Cref{app:properties:ledger} proves this accounting
identity, including the policy-invariant initial interval.

The primary evaluation outcome is full-horizon generalized passenger time per
departed passenger. Using all departed passengers avoids conditioning the
denominator on policy-dependent trip completion. If $N_{\mathrm{dep}}$ and
$N_{\mathrm{cmp}}$ are the numbers of departed and completed passengers in an
episode, define the primary endpoint $Y$, completion rate $R_{\mathrm{cmp}}$,
and unfinished count $U$ by
\begin{equation}
\begin{aligned}
Y&=\frac{\omega_{\mathrm{w}}\int_0^{T_{\mathrm{hor}}}N_{\mathrm{w}}(u)\,\mathrm{d}u
       +\omega_{\mathrm{v}}\int_0^{T_{\mathrm{hor}}}N_{\mathrm{v}}(u)\,\mathrm{d}u}
      {N_{\mathrm{dep}}},\\
R_{\mathrm{cmp}}&=\frac{N_{\mathrm{cmp}}}{N_{\mathrm{dep}}},
\qquad U=N_{\mathrm{dep}}-N_{\mathrm{cmp}}.
\end{aligned}
\label{eq:completion_rate}
\end{equation}
Unfinished passengers remain in the waiting or in-vehicle ledger at the
horizon, so their accrued time is included in the primary endpoint even though
their journeys are incomplete. The primary endpoint also includes the
policy-invariant interval before the first controllable event. That interval is
excluded from the training return because no action can affect it.

\subsection{Target and online-simulator environments}
\label{sec:problem:h2o}

The target evaluation environment $\mathcal{M}_{\mathrm{tgt}}$ is a
high-fidelity SUMO 1.25.0 microsimulation~\citep{lopez2018microscopic}.
The offline replay was collected in its SUMO 1.27.1 counterpart,
$\mathcal{M}_{\mathrm{log}}$, whose transition-duration kernel is denoted
$P_{\mathrm{log}}$. Both SUMO environments use the registered transit
scenario contracts; \Cref{sec:exp:validity} reports the runtime distinction.
The online simulator $\mathcal{M}_{\mathrm{sim}}$ is a calibrated, event-driven
model with the same service, stop, timetable, capacity, passenger journey,
state, action, and reward contracts. It executes transfers causally:
only a passenger's first ride leg is initially available, and each subsequent
leg is released only after the preceding leg actually alights. The simulator
uses calibrated segment runtimes and dwell/boarding mechanics instead of a
microscopic road model.

Their shared observation and action contracts allow us to express the
different dynamics through conditional transition-duration kernels:
\begin{equation}
P_{\mathrm{tgt}}(s',\Delta t\mid s,a)
\quad\text{and}\quad
P_{\mathrm{sim}}(s',\Delta t\mid s,a),
\label{eq:v7_kernels}
\end{equation}
where $s$ and $a$ are a current state and action, and $(s',\Delta t)$ is the
next-state/duration pair; a prime always denotes a successor state.
The two kernels in \Cref{eq:v7_kernels} are denoted
$P_{\mathrm{tgt}}$ and $P_{\mathrm{sim}}$, respectively.
The simulator was accepted only after a preregistered 30-cell fidelity audit
against 30 matched SUMO cells; the experiment section reports the audited
quantities and thresholds.

\begin{figure*}[tp]
\centering
\includegraphics[width=0.70\linewidth]{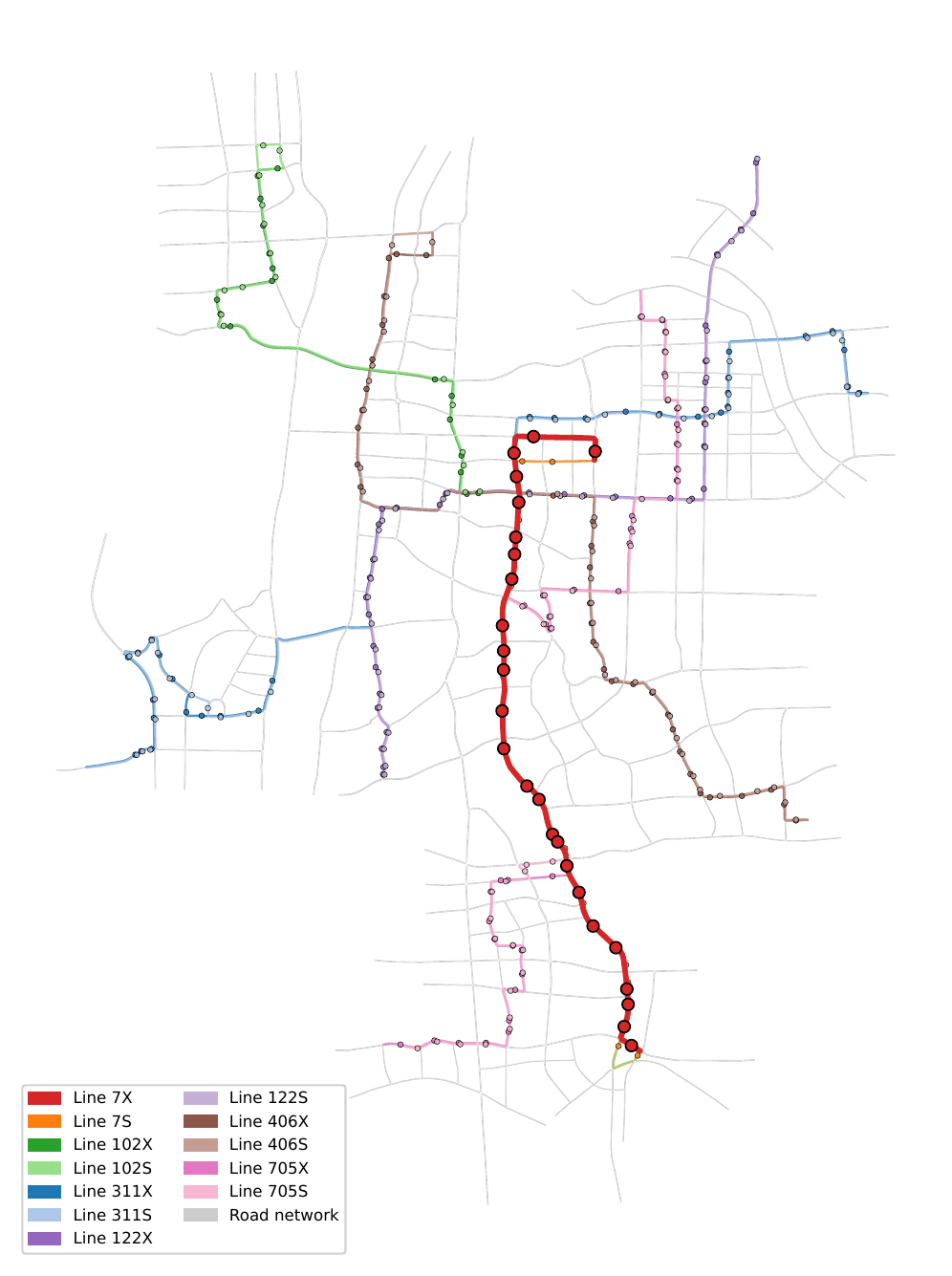}
\caption{Topology of the twelve directional bus services represented in both
the target SUMO environment $\mathcal{M}_{\mathrm{tgt}}$ and the calibrated
online simulator $\mathcal{M}_{\mathrm{sim}}$. The two environments share stop
sequences, timetables, capacities, and passenger origin-destination journeys;
their traffic and event-propagation mechanisms differ. Line 7X is drawn more
heavily only for legibility; policy control covers all 12 services.}
\label{fig:network_topology}
\end{figure*}

\subsection{Offline-to-online learning objective}
\label{sec:problem:objective}

The immutable target replay is
\begin{equation}
\mathcal{D}_{\mathrm{off}}=
\left\{(s_{\iota},a_{\iota},r_{\iota},s'_{\iota},
\Gamma_{\iota},\Delta t_{\iota},z_{\iota})\right\}_{\iota=1}^{N_{\mathrm{off}}},
\qquad N_{\mathrm{off}}=3{,}284{,}718,
\label{eq:v7_dataset}
\end{equation}
In \Cref{eq:v7_dataset}, $\iota$ indexes an offline row,
$N_{\mathrm{off}}$ is the number of rows, and every tuple component follows
\Cref{eq:v7_state,eq:v7_action,eq:v7_discount,eq:v7_reward}. The data comprise
240 complete 24,000-s SUMO episodes: eight behavior policies crossed with 30
registered environment cells formed by ten independent seed blocks and three
demand scales. The behavior policies are zero holding, two random hurdle policies,
four schedule-adjusted headway-balance policies, and a closed-form myopic
welfare policy. They are used for coverage, not as claimed experts.

During learning, interaction is permitted only in
$\mathcal{M}_{\mathrm{sim}}$ and produces simulator replay
$\mathcal{D}_{\mathrm{sim}}$. No target-environment transition is queried
online. Given $\mathcal{D}_{\mathrm{off}}$ and a fixed simulator decision-event
budget, the H2O RL objective is
\begin{equation}
\max_{\pi}\;J_{\mathrm{tgt}}(\pi),
\qquad
J_{\mathrm{tgt}}(\pi)=
\E_{\mathcal{M}_{\mathrm{tgt}},\pi}
\left[\sum_{k=0}^{K-1}r_k\right],
\label{eq:v7_objective}
\end{equation}
where $J_{\mathrm{tgt}}(\pi)$ is the expected target-environment return,
$\E_{\mathcal{M}_{\mathrm{tgt}},\pi}$ takes expectation over the target kernel
and policy, $\pi$ is any admissible policy on the action space $[-1,1]$,
and $K$ is the number of
chronological decision tokens before the horizon. The objective in
\Cref{eq:v7_objective} is evaluated only after training. The learned policy is
$\pi_{\phi}$, with parameter vector $\phi$.
Within a fixed exogenous cell, \Cref{prop:ledger} implies that maximizing the
return and minimizing $Y$ give the same policy ordering. Across cells with
different passenger counts, unnormalized training returns and normalized
evaluation means need not assign the same relative weights to demand levels.
Importantly,
$R_{\mathrm{cmp}}$ in \Cref{eq:completion_rate} is not part of the training
reward; it is a separate service constraint at evaluation. This distinction
allows a policy to reduce accrued generalized passenger time while leaving more
journeys unfinished, and motivates the completion-aware transform in
\Cref{sec:method:reserve}. The setup otherwise matches the H2O setting
introduced by \citet{niu2022h2o} and studied by H2O+~\citep{niu2025h2oplus},
while keeping all evidence explicitly simulator-to-simulator.

\section{Method}
\label{sec:method}

The final controller has two components: a frozen transit-adapted H2O+ actor
and a deterministic physical-action transform. Learning supplies a continuous
state-dependent holding proposal; the transform calibrates its magnitude and
applies a bounded service-completion reserve. No parameter is updated during
target-SUMO evaluation.

\subsection{Structured hurdle actor}
\label{sec:method:architecture}

The actor, critics, and value function use separate structured state encoders
with the same topology. A shared slot network encodes each of the 16 possible
event slots. Five pooled views---the current event, active-slot mean,
active-slot maximum, previously selected-slot mean, and future-slot
mean---are concatenated with the five-value batch prefix and mapped to a
256-dimensional representation. This preserves the ordered action prefix of a
simultaneous batch while masking inactive padding.

Because exact zero hold is frequent and operationally distinct from a small
positive duration, the actor is a hurdle distribution. Let
$p^+_{\phi}(s)=\sigmoid(g_{\phi}(s))$ be the positive-hold
probability, where $g_{\phi}(s)$ is the gate logit and
$\sigmoid(\zeta)=1/(1+\exp(-\zeta))$. Conditional on holding, the latent variable
$\zeta\sim\mathcal{N}(\mu_{\phi}(s),\sigma_{\phi}^2(s))$ follows a normal
distribution with mean $\mu_{\phi}(s)$ and positive standard deviation
$\sigma_{\phi}(s)$. Its positive-branch action is
$a^+=2\sigmoid(\zeta)-1\in(-1,1)$. The policy measure is
\begin{equation}
\pi_{\phi}(\mathrm{d}a\mid s)=
\left(1-p^+_{\phi}(s)\right)\delta_{-1}(\mathrm{d}a)
+p^+_{\phi}(s)f_{\phi}(a\mid s)
\mathbb{I}\{a\in(-1,1)\}\,\mathrm{d}a,
\label{eq:hurdle_policy}
\end{equation}
where $\delta_{-1}$ is the unit point mass at the no-hold action,
$f_{\phi}(a\mid s)$ is the conditional logistic-normal action density,
$\mathbb{I}\{\cdot\}$ is one when its condition holds and zero otherwise,
and $\mathrm{d}a$ denotes the continuous-action measure. For likelihoods and
entropy, we use the density $\varpi_{\phi}(a\mid s)$ with respect to the common
point-mass-plus-Lebesgue measure: it is $1-p^+_{\phi}(s)$ at $a=-1$ and
$p^+_{\phi}(s)f_{\phi}(a\mid s)$ for $-1<a<1$.
\Cref{prop:hurdle} in \Cref{app:properties:hurdle} derives the density and
normalization, and shows that $2\sigmoid(\mu_{\phi}(s))-1$ is the positive
branch's median, not generally its mean.

Deterministic evaluation uses that median when $p^+_{\phi}(s)\geq0.5$ and
otherwise returns $a=-1$. Physical holds are obtained through
\Cref{eq:v7_action} and remain in $[0,60]$ s. The implementation keeps the
continuous branch strictly inside $(-1,1)$ by clipping at floating-point
precision; the operational transform can still apply the exact 60-s cap.

The critic has two independently parameterized heads, with two target heads
updated by Polyak averaging at rate 0.005. The actor, critic, and value heads
each use two 256-unit hidden layers after the structured encoder. Further
implementation values are given in \Cref{app:architecture,app:training}.

\subsection{Offline initialization and simulator continuation}
\label{sec:method:training}

Five actors, with training seeds 6201--6205, were initialized by 50,000 IQL
updates on the immutable target replay~\citep{kostrikov2022iql}. For target
critic minimum $\bar Q_{\min}(s,a)=\min_{i\in\{1,2\}}Q_{\bar\theta_i}(s,a)$
and value function $V_{\nu}(s)$, the advantage
is $A(s,a)=\bar Q_{\min}(s,a)-V_{\nu}(s)$. Here $i$ indexes the two critic
heads, $Q_{\theta_i}$ and $Q_{\bar\theta_i}$ are the online and target
state--action value functions with parameter vectors $\theta_i$ and
$\bar\theta_i$, and $\nu$ is the value-network parameter vector.
All value functions use rewards
scaled by $\kappa_r=10^{-4}$ inside the trainer, not physical passenger-seconds.
The value uses expectile 0.7, and the
actor uses the advantage weight
\begin{equation}
w_{\mathrm{IQL}}(s,a)=
\min\left\{\exp\left(3A(s,a)\right),100\right\}.
\label{eq:iql_weight}
\end{equation}
Here $w_{\mathrm{IQL}}$ is the capped regression weight. The mixed-measure
log likelihood $\log\varpi_{\phi}(a\mid s)$ from \Cref{eq:hurdle_policy} is
used for advantage-weighted actor regression.

After a 5,000-event frozen-policy warmup, training continued in the calibrated
online simulator with 50/50 target and simulator critic minibatches and a
SAC-style actor objective~\citep{haarnoja2018soft,haarnoja2018applications}.
Define the entropy-regularized target-state bootstrap by
\begin{equation}
B_{\mathrm{SAC}}(s')=
\E_{a\sim\pi_{\phi}(\cdot\mid s')}
\left[\bar Q_{\min}(s',a)-\eta_{\mathrm{ent}}
\log\varpi_{\phi}(a\mid s')\right],
\label{eq:sac_bootstrap}
\end{equation}
where $\eta_{\mathrm{ent}}>0$ is the learned entropy temperature and the
expectation is over the hurdle policy. The discrete gate is summed exactly and
the continuous branch uses one reparameterized sample per state. The online Bellman bootstrap
mixed the fixed IQL value and entropy-regularized SAC bootstrap as
\begin{equation}
B_{\mathrm{mix}}(s')=0.1V_{\nu}(s')+0.9B_{\mathrm{SAC}}(s').
\label{eq:mixed_backup}
\end{equation}
Here $B_{\mathrm{mix}}$ is the mixed target-state value and $V_{\nu}$ is
fixed after IQL initialization. The critic target is
\begin{equation}
y_k=\kappa_r r_k+\Gamma_k B_{\mathrm{mix}}(s_{k+1}),
\label{eq:critic_target}
\end{equation}
where $y_k$ denotes a training target, distinct from the physical evaluation
endpoint $Y$. \Cref{eq:critic_target} uses no additional duration discount.
Critics minimize a Huber loss with threshold 1.0 against this target.

The actor and entropy temperature remain frozen through 10,000 simulator
events while critics adapt after warmup. On release, actor and temperature
updates use simulator states only. The actor learning-rate multiplier ramps
from approximately 0.1 to 1 over events 10,000--15,000; the temperature resumes
at its full registered rate. Each subsequent actor update is followed by
parameter interpolation toward the seed-matched 10,000-event actor, with a
fixed interpolation coefficient $5\times10^{-5}$.
The final branches inherited an exact seed-matched optimizer and replay state
at 10,000 simulator events and continued to 25,000 events. From event 15,000
through the final actor update before event 25,000, the preregistered
terminal-quarter branch linearly tapered the actor learning-rate multiplier
from 1 to 0.25 while leaving the remaining update schedule unchanged.
Each seed completed 70,000 learner updates in total. Only the 25,000-event
endpoint was exported for the final policy; checkpoints were not selected by
target-SUMO behavior.

\subsection{Dynamics-ratio diagnostic}
\label{sec:method:ratio}

The implementation retains the H2O+ domain-classifier construction for
diagnosing transition mismatch~\citep{eysenbach2021off,niu2025h2oplus}. A
joint classifier uses $(s,a,s',\Delta t/T_{\mathrm{hor}})$ and a marginal
classifier uses $(s,a)$. With balanced target and simulator classes, their
logit difference estimates the logged-target-to-simulator ratio
\begin{equation}
F_{\psi_J}(s,a,s',\Delta t/T_{\mathrm{hor}})
-G_{\psi_M}(s,a)
\approx
\log\frac{P_{\mathrm{log}}(s',\Delta t\mid s,a)}
          {P_{\mathrm{sim}}(s',\Delta t\mid s,a)}.
\label{eq:ratio_diagnostic}
\end{equation}
Here $F_{\psi_J}$ and $G_{\psi_M}$ are the joint and marginal target-class
logits, with parameter vectors $\psi_J$ and $\psi_M$, respectively; the target
class has label one and the simulator class label zero. The ratio concerns
conditional densities under a common reference measure, and its estimation
requires generalization to the states being evaluated. Because the positive
examples come from $\mathcal{M}_{\mathrm{log}}$, this estimates the logged
runtime's dynamics rather than directly estimating the formal-evaluation
kernel $P_{\mathrm{tgt}}$.
Whole trajectories, not randomly interleaved rows, define held-out validation.
The five final training runs used unit simulator weights in the critic; the
classifier output in \Cref{eq:ratio_diagnostic} remained diagnostic and never
changed a Bellman residual. This matters for interpretation: the formal effect
cannot be attributed to learned density-ratio weighting.

\subsection{Physical action calibration}
\label{sec:method:reserve}

Let $h_{\phi}(s)\in[0,60]$ be the deterministic actor proposal in seconds,
obtained from \Cref{eq:v7_action} before any operational transform, and let
$\rho=t/T_{\mathrm{hor}}$ be the horizon fraction at physical time $t$.
The first transform reduces
aggressive learned holds while retaining a transparent headway response. Let
$H_f,H_b$ be forward and backward headways, $H_f^*,H_b^*$ their targets, and
$I_f,I_b\in\{0,1\}$ indicate whether the corresponding buses are present. The
gain-one headway reserve is
\begin{equation}
h_{\mathrm{hb}}(s)=
\operatorname{clip}\left(
\tfrac{1}{2}\left[I_b(H_b-H_b^*)-I_f(H_f-H_f^*)\right],0,60
\right).
\label{eq:headway_reserve}
\end{equation}
Here $h_{\mathrm{hb}}$ is a headway-based hold in seconds, and
$\clip(\cdot,0,60)$ truncates its scalar argument to
the operational range. Headways and their targets are also measured in seconds.
The calibrated learned action combines the headway response in
\Cref{eq:headway_reserve} with the actor proposal:
\begin{equation}
h_{16}(s)=
\begin{cases}
\min\{h_{\phi}(s),\max[0.125h_{\phi}(s),h_{\mathrm{hb}}(s)]\},
& \rho<0.75\ \text{and}\ h_{\phi}(s)>0,\\
0.125h_{\phi}(s), & \text{otherwise}.
\end{cases}
\label{eq:v16_transform}
\end{equation}
Here $h_{16}$ denotes the inherited v16 calibrated holding proposal.

The completion guard is defined from physical state, not critic value. Let
$n_{\mathrm{on}}$ be onboard passengers, $c$ vehicle capacity, and
$\lambda_{\mathrm{arr}}$ the eligible passenger-arrival rate in passengers
per second. For horizon-fraction endpoints
$0\leq\rho_{\mathrm{lo}}<\rho_{\mathrm{hi}}\leq1$,
\begin{equation}
\mathcal{G}_{[\rho_{\mathrm{lo}},\rho_{\mathrm{hi}})}(s)=
\mathbb{I}\left\{
\rho_{\mathrm{lo}}\leq\rho<\rho_{\mathrm{hi}},\ I_f=1,\ I_b=0,\lambda_{\mathrm{arr}}>0,
\ c>0,\frac{n_{\mathrm{on}}}{c}\leq0.25
\right\}.
\label{eq:completion_guard}
\end{equation}
Thus the binary guard in \Cref{eq:completion_guard} identifies a low-load controlled bus with a vehicle ahead, no
vehicle behind, and continuing eligible demand.

Using the proposal in \Cref{eq:v16_transform}, the direct parent first applies an inherited 20-s tail-service floor only in
the later part of the reserve window:
\begin{equation}
h_{\mathrm{par}}(s)=
\begin{cases}
\max\{h_{16}(s),20\}, & \mathcal{G}_{[0.50,0.75)}(s)=1,\\
h_{16}(s), & \text{otherwise}.
\end{cases}
\label{eq:direct_parent}
\end{equation}
Here $h_{\mathrm{par}}$ is the direct parent's physical hold. The
reserve-augmented candidate changes only the final line of the deployed
behavior:
\begin{equation}
h_{\mathrm{safe}}(s)=
\begin{cases}
\max\{h_{\mathrm{par}}(s),60\}, &
\mathcal{G}_{[0.15,0.75)}(s)=1,\\
h_{\mathrm{par}}(s), & \text{otherwise}.
\end{cases}
\label{eq:completion_reserve}
\end{equation}
Here $h_{\mathrm{safe}}$ is the candidate's physical hold; the subscript names
the registered completion-safety variant. Because 60 s is the physical cap, an active reserve sets the hold to exactly
60 s. The candidate and direct parent share the same actor tensor, training
seed, checkpoint, state, and inherited transforms. They differ only through
\Cref{eq:completion_reserve}. The executed duration $\tau_k$ is
$h_{\mathrm{safe}}(s_k)$ for the candidate and $h_{\mathrm{par}}(s_k)$ for
its parent; the corresponding normalized action is obtained by inverting
\Cref{eq:v7_action}.

\subsection{Development and confirmation separation}
\label{sec:method:separation}

The transform was not derived from the formal outcomes. Successive
preregistered development rounds from v14 through v20 used prior development
results to diagnose actor drift, hold magnitude, headway reserve, and
completion. Several rounds reused the same nine development cells, so these
results are adaptive design evidence rather than confirmation. The 60-s
quarter-load candidate, its direct parent, all actor tensors, formal
environment blocks, primary endpoint, completion margin, and exact sign tests
were frozen before formal outcome access. The formal stage introduced no new
training run and permitted no candidate replacement.

\section{Experiments and Results}
\label{sec:experiments}

The experiment has three evidence layers: offline data and simulator
acceptance, adaptive development, and formal confirmation. Only the last layer
is used for the confirmatory effect claim.

\subsection{Target benchmark and offline data}
\label{sec:exp:benchmark}

The target benchmark uses SUMO on the 12 directional services in
\Cref{fig:network_topology}. Each 24,000-s episode contains the same timetable,
vehicle capacities, road network, and 389 scheduled vehicle trips. Passenger
demand is evaluated at multipliers 0.75, 1.00, and 1.25. Within an environment
block, the three demand levels share one demand-realization seed and one SUMO
seed; departure times receive stable-uniform perturbations of at most 120 s.

The immutable offline file contains 3,284,718 chronological rows from 240
complete episodes. Ten independent environment blocks crossed with three
demand levels and eight behavior policies give 30 exogenous cells and 240
trajectories. Behaviors are zero hold, two random hurdle policies, four
fixed-gain schedule-adjusted headway policies, and a closed-form myopic-welfare
policy. They provide action coverage and are not labelled experts. Data
collection used SUMO 1.27.1 and Python 3.10.12.

The calibrated online simulator shares the timetable, capacity, journey,
state, action, and reward contracts. In a preregistered zero-policy audit over
30 audited cells, every 95\% bootstrap interval stayed within its threshold.
The registered relative tolerances were 3\% for passenger demand, 5\% for
decision count, 15\% for generalized and waiting time, 10\% for in-vehicle
time, and 5\% for completed vehicle-trip duration; completion used an absolute
0.03 tolerance.
The largest point gaps were 4.61\% for generalized passenger time, 8.75\% for
waiting time, 1.66\% for in-vehicle time, 0.43 percentage points for
completion, and 1.04\% for completed vehicle-trip duration. A second audit
crossed the same cells with seven nonzero interventions. The largest absolute
95\% endpoint was 0.00574 for completion, 0.0390 for normalized generalized
passenger-time response, 0.0498 for waiting response, 0.0200 for in-vehicle
response, 0.2078 for relative mean hold, and 0.0117 for vehicle-trip duration;
all were below their preregistered limits. Intervention-response tolerances
were 0.03 for absolute completion response, 0.10 for normalized generalized,
waiting, and in-vehicle responses, and 0.05 for decision and trip-duration
responses. Policy-level mean hold and exact-zero-action rate used relative
0.35 and absolute 0.10 limits, respectively. Both audits used 10,000 bootstrap
resamples of episode-level quantities within each demand level, resampling
SUMO and simulator samples separately. Intervention outcomes were first
differenced against the same environment's zero-hold outcome; time and count
responses were normalized by the corresponding SUMO zero-hold mean. These
acceptance audits precede the independent formal test below.

\subsection{Formal confirmation matrix}
\label{sec:exp:formal_design}

Formal evaluation was endpoint-only: no area under a learning curve (AULC) was
computed or claimed. The sole candidate was the reserve-augmented policy in
\Cref{eq:completion_reserve}. Its direct behavior parent was
\Cref{eq:direct_parent} with the final 60-s reserve disabled. Five
seed-matched actor tensors were fixed before evaluation.

Seven supportive controls were also frozen: the raw v14 terminal-quarter
actor, the base H2O+ endpoint, online-only SAC, delayed-actor v10, actor-hold
v11, actor-ramp v12, and proximal-strong v13. Together with the candidate and
direct parent, nine algorithm labels crossed with five training seeds to form
45 physical behaviors. These controls provide lineage context; they do not
expand the confirmatory claim beyond candidate versus direct parent.

Ten fresh formal environment blocks used SUMO seeds 8101--8110 and demand
realization seeds 8201--8210. None appeared in training or adaptive
development. Crossing 10 blocks, three demand levels, and 45 behaviors produced
1,350 new SUMO episodes and 30 fresh environment cells. Formal execution used
SUMO/libsumo 1.25.0 under Python 3.8.20. All policies in a cell shared the same
exogenous realization.

\subsection{Endpoints and preregistered inference}
\label{sec:exp:statistics}

The primary endpoint is $Y$, accrued generalized passenger seconds per departed
passenger; lower is better. The service endpoint is $R_{\mathrm{cmp}}$, and
unfinished passenger count $U$ is reported descriptively, all as defined in
\Cref{eq:completion_rate}. The independent analysis unit is the formal environment block.
For each candidate--parent comparison, the five training seeds and three demand
levels yield 15 paired observations per block; these are averaged before any
sign test. Let $\beta\in\{1,\ldots,10\}$ index formal blocks,
$\xi\in\{1,\ldots,5\}$ index training seeds, and
$d\in\mathcal{D}_{\mathrm{dem}}=\{0.75,1.00,1.25\}$ denote the demand
multiplier. Superscripts $\mathrm{safe}$ and $\mathrm{par}$ identify the
candidate and parent, respectively. Their primary block difference is
\begin{equation}
\Delta Y_{\beta}=\frac{1}{15}
\sum_{\xi=1}^{5}\sum_{d\in\mathcal{D}_{\mathrm{dem}}}
\left(Y^{\mathrm{safe}}_{\beta,\xi,d}
-Y^{\mathrm{par}}_{\beta,\xi,d}\right).
\label{eq:block_difference}
\end{equation}
The same averaging in \Cref{eq:block_difference} defines
$\Delta R_{\mathrm{cmp},\beta}$ and $\Delta U_{\beta}$. In the tables,
$\Delta Y$, $\Delta R_{\mathrm{cmp}}$, and $\Delta U$ denote mean paired
differences over all cells or the stated demand or seed subset.

Primary superiority uses a one-sided exact paired sign test. Completion uses
the same test after shifting each block difference by the noninferiority margin
$\varepsilon_{\mathrm{NI}}=0.003$: a block is a win when
$\Delta R_{\mathrm{cmp},\beta}>-\varepsilon_{\mathrm{NI}}$.
A primary win requires $\Delta Y_{\beta}<0$. Ties are nonwins. For a win
count $w\in\{0,\ldots,10\}$, the one-sided sign-test probability is
\begin{equation}
p_{\mathrm{sign}}(w)=2^{-10}\sum_{\upsilon=w}^{10}
\binom{10}{\upsilon},
\label{eq:sign_probability}
\end{equation}
where $\upsilon$ is the summation index for possible win counts and
$\binom{10}{\upsilon}$ is the binomial coefficient. At
one-sided significance level $\alpha_{\mathrm{sig}}=0.025$, at least 9 of
10 wins are required by \Cref{eq:sign_probability}
($p_{\mathrm{sign}}(9)=0.010742$). The confirmatory intersection--union test passes
only if both primary superiority and completion noninferiority pass. Mean
direction, demand-stratum direction, result completeness, and parent artifact
and training stability were additional preregistered gates.

\subsection{Confirmed candidate--parent effect}
\label{sec:exp:primary_results}

All 1,350 results and all ten 135-job block manifests passed independent
validation. \Cref{tab:formal_primary} reports the endpoint means. The candidate
reduced the primary endpoint by 92.12 s per departed passenger, a descriptive
relative reduction of 3.71\%. Completion increased by 0.00466 (0.466
percentage points), and unfinished passengers decreased by 77.22 (11.78\%).

\begin{table}[t]
\centering
\small
\caption{Formal endpoint means for the preregistered candidate and its direct
behavior parent. The difference is candidate minus parent.}
\label{tab:formal_primary}
\resizebox{\linewidth}{!}{%
\begin{tabular}{lrrr}
\toprule
Policy & Gen. passenger s/departed & Completion & Unfinished \\
\midrule
Direct parent & 2,484.86 & 0.96185 & 655.42 \\
Reserve-augmented candidate & \textbf{2,392.73} & \textbf{0.96650} & \textbf{578.20} \\
Difference & $-92.12$ & $+0.00466$ & $-77.22$ \\
\bottomrule
\end{tabular}
}
\end{table}

The primary difference was negative in all ten blocks, giving 10 of 10 wins
and a one-sided exact sign-test value of $p_{\mathrm{sign}}(10)=0.0009766$. Completion differences
were positive in every block, so completion noninferiority also achieved 10 of
10 wins and $p_{\mathrm{sign}}(10)=0.0009766$ after the margin shift. The intersection--union test
therefore passed. All preregistered artifact, parent-stability, mean,
demand-stratum, and result-completeness gates passed as well.

The effect was directionally consistent across demand levels
(\Cref{tab:demand_effect}). Its magnitude increased with demand: the primary
reduction ranged from 40.29 s at demand 0.75 to 148.46 s at demand 1.25, while
completion improved at every level. It was also directionally consistent for
each of the five training seeds: primary differences ranged from $-87.25$ to
$-97.44$ s and completion differences from $+0.00444$ to $+0.00485$
(\Cref{app:seed_effect,tab:seed_effect}).

\begin{table}[t]
\centering
\small
\caption{Candidate-minus-parent formal differences by demand multiplier.
Negative generalized time and unfinished counts are favorable; positive
completion is favorable.}
\label{tab:demand_effect}
\begin{tabular}{lrrr}
\toprule
Demand & $\Delta Y$ (s/departed) & $\Delta R_{\mathrm{cmp}}$ & $\Delta U$ (passengers) \\
\midrule
0.75 & $-40.29$ & $+0.00185$ & $-21.00$ \\
1.00 & $-87.61$ & $+0.00470$ & $-70.92$ \\
1.25 & $-148.46$ & $+0.00741$ & $-139.74$ \\
\bottomrule
\end{tabular}
\end{table}

\subsection{Supportive lineage controls}
\label{sec:exp:controls}

\Cref{tab:formal_controls} gives descriptive means for all fixed algorithm
labels. The reserve-augmented candidate had lower generalized passenger time and
higher completion than every listed control, including online-only SAC. Its
primary block difference versus each control was favorable in all ten blocks.
These comparisons were prespecified robustness gates, but the controls do not
share the candidate's exact behavior parent and are not part of the
intersection--union confirmatory claim.

\begin{table*}[t]
\centering
\small
\caption{Descriptive formal means for the fixed candidate, direct parent, raw
actor, and lineage controls. Bold denotes the best displayed mean.}
\label{tab:formal_controls}
\resizebox{\textwidth}{!}{%
\begin{tabular}{lrrr}
\toprule
Algorithm label & Gen. passenger s/departed & Completion & Unfinished passengers \\
\midrule
Reserve-augmented candidate & \textbf{2,392.73} & \textbf{0.96650} & \textbf{578.20} \\
Direct behavior parent & 2,484.86 & 0.96185 & 655.42 \\
Raw v14 terminal-quarter actor & 2,623.19 & 0.96453 & 610.03 \\
Base H2O+ & 2,696.98 & 0.96288 & 637.41 \\
Online-only SAC & 2,593.85 & 0.96232 & 648.31 \\
Delayed actor (v10) & 2,532.20 & 0.95894 & 702.57 \\
Actor hold (v11) & 2,569.68 & 0.96023 & 681.31 \\
Actor ramp (v12) & 2,585.56 & 0.96194 & 653.26 \\
Proximal strong (v13) & 2,581.26 & 0.96053 & 676.66 \\
\bottomrule
\end{tabular}
}
\end{table*}

\subsection{Data and environment diagnostics}
\label{sec:exp:validity}

The artifact checks found no nonfinite rows, no passenger-ledger reconstruction
error, and 240 terminal rows for 240 trajectories. Nevertheless, four factors
qualify algorithmic interpretation. \Cref{app:support} provides the support
counts and dataset-integrity details.

First, the 3.28 million serialized transitions arise from only 30 independent
exogenous cells, each reused by eight behavior policies. Statistical evidence
must therefore be organized by environment block rather than transition.
Second, 80.2246\% of offline actions are exact zero. The final low-load guard
appears in 3.91\% of offline rows, and only 3.69\% of those guard-state actions
are at least 30 s. The 60-s reserve should consequently be interpreted as a
bounded operational overlay, not a behavior learned from dense action support.

Third, the ratio diagnostic generalized poorly. Mean joint classifier area
under the receiver-operating-characteristic curve (AUC) was 0.9477 on training data
and 0.4520 on trajectory-held-out validation data; mean raw validation
effective sample size (ESS) fraction was 0.0615. \Cref{app:ratio,tab:ratio_diag}
reports the across-seed diagnostic summary. Unit simulator weights prevented
this failure from changing the sealed optimizer, but no dynamics-correction
effect is established. Fourth, offline collection used SUMO 1.27.1 whereas
formal evaluation used SUMO/libsumo 1.25.0. The fresh formal result therefore
remains favorable in the evaluation runtime, but a paired version-sensitivity
experiment was not conducted. Canonical feature clipping was
otherwise limited: only 1,565 rows (0.0476\%) exceeded the backward-headway
cap, and the other audited caps had zero exposure.

\subsection{Execution and reproducibility}
\label{sec:exp:reproducibility}

Each formal result binds the method identity, training seed, actor tensor,
checkpoint, target dataset, evaluator source, scenario, and runtime. Ten matrix
manifests account for all 1,350 files. A frozen analysis produced the reported
statistics, and a separate implementation independently recomputed the result
count, means, block wins, sign-test value, and all preregistered gates. The
closeout also verified that no local SUMO execution occurred during analysis
and that formal result files were not modified.

Only result JavaScript Object Notation (JSON) files, matrix manifests, compact logs, and verification
metadata were synchronized from the evaluation server. Checkpoints, datasets,
optimizer state, replay buffers, and remote workspaces were not duplicated in
the paper artifact. This preserves a minimal reproducibility package while the
frozen deployment manifest records the larger server-side dependencies.
\Cref{app:evidence} identifies the compact analysis and verification files.

\section{Discussion}
\label{sec:discussion}

\subsection{What the formal result establishes}

The formal comparison isolates one mechanism. Candidate and parent share the
same five actor tensors, checkpoints, state, and inherited action transform;
only the 60-s completion reserve differs. On ten fresh environment blocks, the
reserve reduced generalized passenger time and improved completion in every
block. The direction also held at every demand level and training seed. This is
strong evidence that the fixed reserve improves this simulator-to-simulator
endpoint relative to its direct parent under the registered scenario family.

The demand pattern is operationally coherent. The benefit is smallest at
demand 0.75 and largest at demand 1.25, where leaving a tail without a follower
has greater passenger consequences. The reserve can add local in-vehicle time,
but it may also prevent a low-load tail bus from advancing in a way that leaves
later demand unserved before the horizon. The formal decrease in unfinished
passengers is consistent with the latter mechanism in these scenarios.
This interpretation is mechanistic rather than causal at the level of
individual passengers because the experiment did not randomize guard
activations within a trajectory.

The result does not establish an active density-ratio benefit. Classifier
validation was near chance, and final training used unit simulator weights.
The learned component should therefore be described as an IQL-initialized,
simulator-continued H2O+ actor with mixed replay, not as a successfully
ratio-corrected policy. The confirmed contribution is the composition of this
fixed actor with the explicit completion-safety reserve.

\subsection{Why development evidence alone was insufficient}

The development history was adaptive. Successive rounds diagnosed actor drift,
hold scale, headway reserve, tail service, and completion, and several rounds
reused the same nine development cells. A favorable result at one stage could
therefore reflect both a real mechanism and accumulated adaptation to those
cells. Earlier variants also used different service margins in some analyses;
passing a looser completion boundary does not imply passing the final 0.003
margin. For these reasons, development improvements were not treated as a
formal effect.

The ten-block confirmation was designed to break that feedback loop. Candidate,
parent, environment cells, endpoints, margin, and tests were frozen before
outcome access. Formal cells were disjoint from training and all development
rounds, and candidate replacement was forbidden. This design, rather than the
number of development iterations, is what supports the bounded confirmed
claim.

\subsection{Limitations}
\label{sec:discussion:limitations}

\paragraph{Data and state representation.}

The offline dataset is large in rows but modest in independent exogenous
variation. Simultaneous decisions and long trajectories create millions of
correlated SMDP tokens from only 30 environment cells. The five training seeds
represent optimizer variation conditional on that fixed data, not five new
datasets. Future data collection should prioritize more independent days,
demand realizations, and disruption patterns instead of denser serialization of
the same cells. The centralized observation also compresses passenger
histories and future service availability. Its Markov sufficiency is a
modeling assumption, not a proved property of the 229-feature encoding.
\Cref{prop:ledger,prop:hurdle} establish exact accounting and action
normalization, not convergence or a completion guarantee.

Positive-action support is also sparse. Exact zero hold dominates replay, and
long holds are uncommon in the final guard states. The completion reserve is
therefore intentionally a transparent physical rule rather than a claim that
the offline actor learned reliable values for 60-s actions in these states. A
future learned completion controller would need explicit tail-state/action
coverage and a constrained or terminal-shortfall objective. Training also uses
total accrued cost, whereas evaluation normalizes each episode by departed
passengers. Their within-cell ordering agrees by \Cref{prop:ledger}, but
cross-demand weighting can differ.

\paragraph{Simulator and runtime scope.}
Training data were collected with SUMO 1.27.1, whereas development and formal
evaluation were locked to SUMO/libsumo 1.25.0. Because formal performance is
measured entirely in 1.25.0, the reported comparison itself is internally
paired. The shift nevertheless complicates attribution between algorithm,
simulator mismatch, and version-specific behavior. A version-sensitivity matrix
with the same fixed policies in both runtimes is needed before claiming
runtime-invariant transfer.

The reduced-order online simulator passed zero-policy and intervention-response
audits, but those tests do not cover every state distribution induced by a
learned policy. It also approximates transfer-service choice through calibrated
frequencies rather than a dynamic first-arrival model. Both environments share
the same timetable, demand source, capacities, and passenger contract, so the
study does not test changes in traveler preferences, fares, dispatch rules, or
fleet composition.

\paragraph{Statistical and comparison scope.}

The exact sign test does not assume normally distributed effect sizes; it
uses independent blocks and a null directional win probability no greater
than one half. It reports directional
consistency, not a confidence interval for a population-average effect. Ten
blocks also provide limited resolution: the smallest attainable one-sided
value is $2^{-10}$. More independent blocks would support effect-size
uncertainty and interaction analyses without relying on episode-level
pseudoreplication.

The lineage-control table is useful context but not a set of confirmatory
pairwise claims. Only the direct parent removes the final reserve while holding
the underlying actor and prior transform fixed. The study is also endpoint-only;
it cannot support a statement about learning speed, AULC, or compute efficiency.
Likewise, the formal result does not retroactively validate unsuccessful
development candidates or the ratio diagnostic.

\paragraph{Operational scope.}

The reserve is interpretable and bounded, but it is not a complete deployment
safety system. The 60-s cap does not encode terminal layovers, driver rules,
accessibility obligations, connection protection, dispatch authority, or
incident response. Before a field trial, the fixed policy should be calibrated
and evaluated on temporally held-out Automatic Vehicle Location, Automatic
Passenger Counting, and fare-collection records, followed by shadow-mode
testing with dispatcher review.

The current evidence ends at a provenance-locked simulator-to-simulator
confirmation. It supports writing and reporting the bounded result, but not a
claim of live passenger benefit. A field study would require separate
authorization, operational stopping rules, and independent evaluation days.

\section{Conclusion}
\label{sec:conclusion}

We studied cross-fidelity offline-to-online RL for asynchronous multi-line bus
holding under an exact chronological passenger-time contract. The final
controller combines a fixed IQL-initialized, simulator-continued H2O+ actor with
a deterministic completion-safety reserve. Relative to the identical direct
parent without the final reserve, the fixed candidate reduced generalized
passenger time by 92.12 s per departed passenger, increased completion by
0.00466, and reduced unfinished passengers by 77.22. These paired effects were
confirmed within a 1,350-episode fresh formal SUMO matrix. Primary superiority and completion noninferiority passed in all
ten independent blocks under the preregistered intersection--union test.

The completion-aware reserve
improves the fixed parent policy at the registered endpoint in the evaluated
SUMO scenario family. Future work should align
training and evaluation runtimes, collect more independent environment cells
with richer positive-hold support, incorporate completion into the learning
objective, and validate the fixed policy against held-out operational records.

\section*{Data and Code Availability}

The reproducibility package contains the frozen training and evaluation source,
scenario and method manifests, 1,350 formal result records, the registered
analysis, and an independent analysis and closeout verification. It excludes
duplicated checkpoints, optimizer states, replay buffers, and server
workspaces; these are bound through deployment manifests rather than copied
into the paper archive. The project source code is available at
\url{https://github.com/erzhu419/sumo-rl}. The final sharing terms for operating data and derived
SUMO inputs require author and data-provider confirmation and will be inserted
before submission.

\section*{Declaration of Competing Interest}

A competing-interest declaration requires confirmation from all authors and
will be finalized before submission.

\section*{Funding}

Funding information and author-identifying acknowledgements are withheld during
double-anonymous review and require author confirmation before submission.

\section*{Declaration of Generative Artificial Intelligence (AI) and AI-Assisted Technologies in the Manuscript Preparation Process}

During preparation of this work, the authors used AI-assisted tools for
manuscript drafting, language editing, reference verification, LaTeX
consistency checks, and code and log inspection. The
authors reviewed and verified the resulting text and take full responsibility
for the article. No generative-AI-created figure, image, or artwork is included
in the manuscript.

\section*{Contributor Roles Taxonomy (CRediT) Authorship Contribution Statement}

Author contributions are withheld during double-anonymous review.

\clearpage
\appendix
\numberwithin{equation}{section}
\numberwithin{table}{section}
\numberwithin{figure}{section}
\renewcommand{\theequation}{\Alph{section}.\arabic{equation}}
\renewcommand{\thetable}{\Alph{section}.\arabic{table}}
\renewcommand{\thefigure}{\Alph{section}.\arabic{figure}}
\section{Accounting and action-distribution properties}
\label{app:properties}

\subsection{Chronological passenger-time accounting}
\label{app:properties:ledger}

The following identity justifies the reward accounting in
\Cref{sec:problem:reward} and its relationship to the reported endpoint.

\begin{proposition}[Nonduplicating AEC ledger]
\label{prop:ledger}
Let $t_0$ be the first controllable batch time, and suppose the positive-duration
batch transitions partition $[t_0,T_{\mathrm{hor}}]$. Assume the passenger-count
functions are integrable. With the serialization in \Cref{sec:problem:smdp}
and costs in \Cref{eq:v7_cost},
\begin{equation}
\sum_{k=0}^{K-1}C_k
=\omega_{\mathrm{w}}\int_{t_0}^{T_{\mathrm{hor}}}N_{\mathrm{w}}(u)\,\mathrm{d}u
+\omega_{\mathrm{v}}\int_{t_0}^{T_{\mathrm{hor}}}N_{\mathrm{v}}(u)\,\mathrm{d}u.
\label{eq:ledger_identity}
\end{equation}
Define the pre-control cost
\begin{equation}
C_{\mathrm{pre}}=
\omega_{\mathrm{w}}\int_0^{t_0}N_{\mathrm{w}}(u)\,\mathrm{d}u
+\omega_{\mathrm{v}}\int_0^{t_0}N_{\mathrm{v}}(u)\,\mathrm{d}u.
\label{eq:precontrol_cost}
\end{equation}
For $N_{\mathrm{dep}}>0$, the primary endpoint satisfies
\begin{equation}
Y=\frac{C_{\mathrm{pre}}+\sum_{k=0}^{K-1}C_k}{N_{\mathrm{dep}}}.
\label{eq:endpoint_ledger}
\end{equation}
For a fixed exogenous realization, $C_{\mathrm{pre}}$ and $N_{\mathrm{dep}}$
are policy-invariant under the registered demand contract.
\end{proposition}

\begin{proof}
Every nonfinal token in a simultaneous batch has $\Delta t_k=0$, so both
integrals defining its cost vanish. The final token carries exactly the
interval from that batch to the next physical batch or the horizon.
These intervals do not overlap except at endpoints of zero measure. Finite
additivity of the integrals therefore gives \Cref{eq:ledger_identity}.
Adding the initial interval in \Cref{eq:precontrol_cost} gives the full-horizon
integrals in \Cref{eq:completion_rate}, proving \Cref{eq:endpoint_ledger}.
There is no control before $t_0$, and passenger departures are fixed by the
exogenous demand realization, which proves the final statement.
\end{proof}

\newpage
\subsection{Hurdle density and deterministic action}
\label{app:properties:hurdle}

\Cref{sec:method:architecture} uses the following distributional property to
compute likelihoods and specify deterministic evaluation.

\begin{proposition}[Normalized hurdle policy and positive-branch median]
\label{prop:hurdle}
For finite $g_{\phi}(s)$ and $\mu_{\phi}(s)$ and
$\sigma_{\phi}(s)>0$, let $\mathfrak{n}_{\phi}(\cdot\mid s)$ denote the
normal density with mean $\mu_{\phi}(s)$ and standard deviation
$\sigma_{\phi}(s)$. The conditional positive-branch density is
\begin{equation}
f_{\phi}(a\mid s)=
\frac{2\mathfrak{n}_{\phi}(\zeta(a)\mid s)}{1-a^2},
\qquad
\zeta(a)=\log\frac{1+a}{1-a},\quad -1<a<1,
\label{eq:positive_density}
\end{equation}
where $\zeta(a)$ is the latent coordinate associated with action $a$.
The policy in \Cref{eq:hurdle_policy} is a probability measure, and its
conditional positive-branch median is $2\sigmoid(\mu_{\phi}(s))-1$.
\end{proposition}

\begin{proof}
The map $a=2\sigmoid(\zeta)-1$ is a strictly increasing bijection from
$\R$ to $(-1,1)$, with derivative
$\mathrm{d}a/\mathrm{d}\zeta=(1-a^2)/2$. Change of variables gives
\Cref{eq:positive_density} and
$\int_{-1}^{1}f_{\phi}(a\mid s)\,\mathrm{d}a=1$.
The continuous mass is therefore $p^+_{\phi}(s)$; adding the point mass
$1-p^+_{\phi}(s)$ yields total mass one. The normal latent variable has median
$\mu_{\phi}(s)$, and a strictly increasing transformation preserves its
median. This proves the deterministic-action statement; nonlinear
transformation does not generally preserve the conditional mean.
\end{proof}

\section{Implementation and provenance details}
\label{app:impl}

\subsection{Structured network implementation}
\label{app:architecture}

The 229-dimensional canonical state has a five-value batch prefix followed by
16 padded 14-value event slots. Each slot contains an active flag, 12
normalized local features, and the stored prior action. The service index is
expanded to a 12-way one-hot code; selected/current status and relative slot
position are appended before encoding. One shared slot multilayer perceptron
maps each local vector to 128 features. The current slot, active-slot mean,
active-slot maximum, previously selected-slot mean, and future-slot mean form
five pooled views. A fusion multilayer perceptron combines these views with the
five-value prefix to produce 256 state features.

The actor maps the fused features through two 256-unit hidden layers to a gate
logit, latent normal mean, and log standard deviation. The
log standard deviation is clipped to $[-5,2]$. The initial gate bias gives a
positive-hold probability of approximately 0.35 and the conditional bias
favors short holds. The critic ensemble contains two independently
parameterized heads after one structured encoder; the target ensemble is an
exact structural copy. The value network has a separate encoder and scalar
head. Main modules use rectified linear unit activations and no layer
normalization.

\subsection{Frozen training values}
\label{app:training}

\Cref{tab:hparams} lists the principal final-profile values. All optimizers use
adaptive moment estimation (Adam). Stored rewards remain physical generalized passenger-seconds; numerical
conditioning is applied only inside the trainer. Actions, rewards,
continuations, durations, and terminal flags remain single precision.

\begin{table*}[tp]
\centering
\small
\caption{Final actor-training and architecture values.}
\label{tab:hparams}
\begin{tabular}{@{}p{0.54\linewidth}p{0.42\linewidth}@{}}
\toprule
Parameter & Value \\
\midrule
Canonical state width & 229 \\
Maximum event slots & 16 \\
Slot embedding and fused widths & 128 and 256, respectively \\
Actor/critic/value hidden widths & 256, 256 \\
Critic heads & 2 \\
Maximum hold & 60 s \\
Critic and simulator-actor minibatches & 256 and 128, respectively \\
Actor/critic/value/entropy learning rates & $3\times10^{-4}$ \\
Ratio learning rate & $10^{-4}$ \\
Target update rate & 0.005 \\
Simulator critic-minibatch fraction & 0.5 \\
IQL expectile & 0.7 \\
IQL exponential multiplier & 3 \\
Maximum IQL actor weight & 100 \\
Mixed-backup coefficient & 0.1 \\
Initial entropy temperature & 1.0 \\
Target entropy & 0.0 \\
Trainer reward scale $\kappa_r$ & $10^{-4}$ \\
Critic Huber threshold & 1.0 \\
Actor interpolation coefficient & $5\times10^{-5}$ \\
Offline IQL updates & 50,000 \\
Frozen-policy warmup & 5,000 events \\
Actor/temperature freeze endpoint & 10,000 events \\
Actor ramp & 10,000--15,000 events; multiplier approximately 0.1 to 1 \\
Ratio pretraining & 1,000 updates \\
Ratio update interval & 4 learner updates \\
Terminal actor taper & 15,000--25,000 events; multiplier linear 1 to 0.25 \\
Total learner updates & 70,000 \\
Fixed actor endpoint & 25,000 simulator events \\
Training seeds & 6201--6205 \\
\bottomrule
\end{tabular}
\end{table*}

The terminal-quarter branch reused each seed's sealed 10,000-event state and
continued without restarting optimizer, replay, or simulator chronology. All
15 terminal-taper branches passed artifact and training-stability checks; only
the preregistered quarter-rate branch supplied the five actors used by the
final candidate. The actor regression uses the weight in
\Cref{eq:iql_weight}; the online critic uses \Cref{eq:critic_target} with
\Cref{eq:mixed_backup,eq:sac_bootstrap}.

\subsection{Ratio diagnostic}
\label{app:ratio}

Target and simulator examples enter each domain-classifier loss in equal
numbers. Offline validation is split by complete environment block, and
simulator validation holds out complete trajectories. A diagnostic random
stream is separate from optimization sampling. \Cref{tab:ratio_diag} reports
the across-seed summary that determined interpretation of the ratio module.

\begin{table}[t]
\centering
\small
\caption{Domain-ratio diagnostics across five sealed actor-training seeds.}
\label{tab:ratio_diag}
\begin{tabular}{lrrr}
\toprule
Diagnostic & Mean & Minimum & Maximum \\
\midrule
Training joint AUC & 0.9477 & 0.9370 & 0.9580 \\
Held-out joint AUC & 0.4520 & 0.4366 & 0.4747 \\
Held-out raw ESS fraction & 0.0615 & 0.0100 & 0.1530 \\
\bottomrule
\end{tabular}
\end{table}

The separation between training and validation AUC, together with low raw
ESS, shows that the classifier did not supply a usable
generalizing correction. The optimizer used unit ratio weights, so the
diagnostic failure did not directly perturb critic updates.

\subsection{Dataset integrity and support}
\label{app:support}

The immutable replay contains 3,284,718 rows, 240 trajectories, 240 terminal
rows, and 1,025,206 zero-duration agent-environment-cycle rows. No nonfinite row
was found, and chronological passenger-ledger reconstruction had zero maximum
error. Exact zero hold accounts for 80.2246\% of actions, while 2,145 rows use
the exact 60-s upper endpoint. The audit groups uncertainty by complete
environment cells rather than treating rows as independent samples.

For the final low-load guard, 128,437 rows (3.91\%) satisfy the physical-state
conditions in the audited replay. Among these rows, 5.81\%, 3.69\%, and 2.17\%
have behavior holds of at least 20, 30, and 40 s, respectively. These values
motivate reporting the final 60-s action as an operational reserve outside
dense logged support.

\subsection{Seed-specific formal effects}
\label{app:seed_effect}

\Cref{tab:seed_effect} shows the candidate-minus-parent effect after averaging
each training seed over the ten formal blocks and three demand levels. The
direction is the same for all seeds.

\begin{table}[t]
\centering
\small
\caption{Formal candidate-minus-parent differences by training seed.}
\label{tab:seed_effect}
\begin{tabular}{lrrr}
\toprule
Seed & $\Delta Y$ (s/departed) & $\Delta R_{\mathrm{cmp}}$ & $\Delta U$ (passengers) \\
\midrule
6201 & $-87.25$ & $+0.00457$ & $-75.43$ \\
6202 & $-97.44$ & $+0.00485$ & $-80.67$ \\
6203 & $-94.28$ & $+0.00478$ & $-79.63$ \\
6204 & $-93.44$ & $+0.00465$ & $-76.97$ \\
6205 & $-88.21$ & $+0.00444$ & $-73.40$ \\
\bottomrule
\end{tabular}
\end{table}

\subsection{Formal evidence package}
\label{app:evidence}

The formal package records 45 physical behaviors, 30 fresh environment cells,
ten complete block manifests, and 1,350 result records. The registered analyzer
reads only endpoint results and implements the exact sign tests and gates in
\Cref{sec:exp:statistics}. A second analyzer independently recomputes result
counts, candidate--parent means, demand and block differences, win counts, and
gate outcomes. A separate closeout verifies source, protocol, environment,
method, scheduler, and file-manifest bindings.

The central machine-readable files are:
\begin{itemize}[leftmargin=*,nosep]
    \item \path{formal_confirmation_analysis_v1.json};
    \item \path{formal_confirmation_analysis_verification_v1.json};
    \item \path{evaluation_closeout_v1.json}; and
    \item \path{evaluation_closeout_verification_v1.json}.
\end{itemize}
\begin{samepage}
They are stored in the following directory:
\begin{quote}
\small
\texttt{reproducibility/}\\
\texttt{h2oplus\_v20\_formal\_confirmation\_results/}
\end{quote}
The synchronized package intentionally contains no checkpoint, optimizer,
dataset, replay, or remote-workspace copy.
\end{samepage}

\bibliographystyle{elsarticle-harv}
\bibliography{references}

\end{document}